\documentclass[12pt]{article}
\usepackage{microtype}
\usepackage{graphicx}
\usepackage{subfigure}
\usepackage{booktabs} 
\usepackage{amsfonts,amscd,amssymb}
\usepackage{url}
\usepackage{amsthm,amsmath,natbib}
\usepackage{hyperref}
\usepackage{natbib}

\newtheorem{Theorem}{Theorem}[section]

\newtheorem{Remark}[Theorem]{Remark}

\newcommand{\mc}{\mathcal}

\newcommand{\Fcal}{\mathcal{F}}

\newcommand{\Mcal}{\mathcal{M}}

\newcommand{\Ebb}{\mathbb{E}}

\newcommand{\Pbb}{\mathbb{P}}

\newcommand{\Rbb}{\mathbb{R}}

\newcommand{\Zbb}{\mathbb{Z}}

\newtheorem{definition}{Definition}
\newtheorem{theorem}{Theorem}
\newtheorem{lemma}{Lemma}
\renewcommand{\phi}{\varphi}

\begin{document}
	
\title{\textbf{A Generalization Bound of Deep Neural Networks for Dependent Data}}

\author{\small\sc Quan Huu Do$^a$, Binh T. Nguyen$^{a,b}$, Lam Si Tung Ho$^{c}$}
\date{}
\maketitle
\thispagestyle{empty}

\begin{center}
\textit{\small a-University of Science, Vietnam National University Ho Chi Minh City, Viet Nam\\
b-AISIA Research Lab, Ho Chi Minh City, Vietnam\\
c-Dalhousie University, Halifax, Nova Scotia, Canada}
\end{center}

\hrule
\begin{abstract}
Existing generalization bounds for deep neural networks require data to be independent and identically distributed (iid).
This assumption may not hold in real-life applications such as evolutionary biology, infectious disease epidemiology, and stock price prediction. 
This work establishes a generalization bound of feed-forward neural networks for non-stationary $\phi$-mixing data.

\end{abstract}

\section{Introduction}

Explaining the generalization ability of machine learning methods (that is, they can provide a close fit to new, unseen data) lies at the heart of theoretical machine learning.
The main direction for this research topic is to bound the difference between the expected loss (population loss) and the empirical loss (training loss). 
This is known as generalization bound, which has been studied extensively in various settings \citep{freund2004generalization, zou2009generalization, agarwal2012generalization, cuong2013generalization, bartlett2017spectrally, golowich2018size, lugosi2022generalization}.

In the last decade, deep neural networks have become the central attention of the machine learning community due to their remarkable success in solving complex tasks that are considered to be challenging for existing machine learning methods. 
For example, in computer vision, tasks like image classification, facial recognition, and object detection have significant progress by applying deep neural networks \citep{krizhevsky2012imagenet}. In natural language processing, deep learning models have become state-of-the-art in language translation, sentiment analysis, and chatbots \citep{vaswani2017attention}. 
Additionally, they have made undeniable contributions to fields beyond computer sciences, including autonomous vehicles, healthcare \citep{esteva2019guide}, and finance \citep{heaton2017deep}.

Effort has been made to derive the generalization bound for neural networks \citep{bartlett2017spectrally, golowich2018size, dinh2020consistent, ho2022searching}.
However, these results assume that data are independent and identically distributed (iid).
Unfortunately, this assumption is not often satisfied in many applications, including evolutionary biology, infectious disease epidemiology, and stock price prediction. 
Therefore, it is crucial to study the generalization ability of deep neural networks when data are not iid.
In this paper, we will bridge this gap by establishing a generalization bound of feed-forward neural networks for non-stationary $\phi$-mixing data.
It is worth noticing that mixing data is the most common alternative to iid data \citep[e.g.][]{white1984nonlinear, modha1996minimum, mohri2010stability, dinh2015learning, ho2020recovery}.
In this paper, we consider on the popular $\phi$-mixing sequences: data are dependent, but the dependency of two data points decreases as their distance increases.
Furthermore, we do not require data to be identically distributed.
Instead, we allow the marginal distribution of data to converge to an unknown target distribution.
Under this setting, we establish a new generalization bound for feed-forward neural networks.

\section{Setting and main results}
\noindent\textbf{Setting:} 
We consider a classification problem setting where the input-output pairs $\{(X_i, Y_i)\}_{i=1}^n \subset \mathbb{R}^d \times \{1, 2, \ldots, K\}$ are not i.i.d. 
Specifically, we relax the independence assumption by assuming that the data $\mathcal{Z} = \{(X_i, Y_i)\}_{i=1}^n$ are generated from a $\phi$-mixing sequence:

\begin{definition}
    Let the $\{Z_k\}_{k=0}^\infty$ be a sequence of random variables. For any $i,j\in\Zbb$, let $\sigma_i^j$ denote the $\sigma$-algebra generated by the random variables $\{Z_k\}_{k=i}^j$. Then, for any positive integer $k$, the $\phi$-mixing coefficients of the stochastic process $Z$ is defined as 
    $$\mathit{\phi}(k)=\sup_{\substack{n\\A\in\sigma_{n+k}^\infty\\B\in\sigma_{0}^n}}\big|\Pbb\left[A|B\right]-\Pbb[A]\big|.$$
The sequence of variables $\{Z_k\}_{k=0}^\infty$ is said to be $\phi$-mixing if $\phi(k)\rightarrow0$ as $k\rightarrow\infty$.
\end{definition}

Additionally, we assume that the data $\mathcal{Z} = \{(X_i, Y_i)\}_{i=1}^n$ are not identically distributed.
Instead, the marginal distribution of $(X_i, Y_i)$ converges to the target distribution $\Pi$, which is the marginal distribution of the test data.
More precisely, 
\[
\mu_n := \big|\big|\Pbb_n-\Pi\big|\big|_{\text{TV}} \to 0
\]
where $\Pbb_n$ is the marginal distribution of $(X_i, Y_i)$ and $\| \cdot \|_{\text{TV}}$ is the total variation distance.\\

This paper will focus on feed-forward neural networks with $L$ hidden layers where the $i$-th layer has a weight matrix $A_i$ and an activation function $\sigma_i$.
Throughout the paper, we assume that each weight matrix has a dimension at most $W$ along each axis. 
Moreover, the activation functions $\sigma_i$ is $p_i$-Lipschitz (i.e. $|\sigma_i(x)-\sigma_i(y)| \leq p_i | x - y|$ for all $x,y \in \mathbb{R}$) and $\sigma_i(0) = 0$. 

Denote $\mathcal{A}=(A_1,\ldots,A_L)$ and $\boldsymbol{\sigma} = (\sigma_1,\ldots,\sigma_L)$. 
The corresponding feed-forward neural network $F_{\mathcal{A}, \boldsymbol{\sigma}}$ is
\begin{equation*}
    F_{\mathcal{A}, \boldsymbol{\sigma}}(x)=\sigma_L(A_L(\sigma_{L-1}(A_{L-1} \ldots \sigma_1(A_1x) \ldots)).
\end{equation*}
The network output $F_{\mathcal{A}, \boldsymbol{\sigma}}(x)\in\Rbb^K$ is converted to a class label in $\{1, \ldots , K\}$
by taking the $\arg\max$ over components, with an arbitrary rule for breaking ties. 
We will work with the popular ramp loss $\ell_\gamma : \mathbb{R} \to \mathbb{R}^+$: $\ell_\gamma(r):= (1+r^-/\gamma)^+$ where $a^+ = \max \{a , 0\}$ and $a^- = \min \{a, 0\}$.
The empirical loss $\mathcal{L}_\mathcal{Z}(F_{\mathcal{A}, \boldsymbol{\sigma}})$ and expected loss $\mathcal{L}(F_{\mathcal{A}, \boldsymbol{\sigma}})$ are defined as
\begin{align*}
   \mathcal{L}_\mathcal{Z}(F_{\mathcal{A}, \boldsymbol{\sigma}}) &= \frac{1}{n}\sum_{i=1}^n\ell_\gamma(-\mathcal{M}(F_{\mathcal{A}, \boldsymbol{\sigma}}(X_i),Y_i)) \\
   \mathcal{L}(F_{\mathcal{A}, \boldsymbol{\sigma}}) &= \mathop{\Ebb}_{(X,Y)\sim \Pi}\left[\ell_\gamma(-\mathcal{M}(F_{\mathcal{A}, \boldsymbol{\sigma}}(X),Y))\right]
\end{align*}
where $\Mcal(v,j):=v_j-\max_{i\neq j}v_i$ is the margin operator.

\noindent\textbf{Main results:} 
First, we will derive a uniform bound of the gap between expected loss and empirical loss for a general hypothesis space $\mathcal{H}$ and a bounded loss $\ell$ using Rademacher complexity. 

\begin{definition}
    Given a class of function $\mathcal{F}$ and a data set $\mathcal{Z}=(Z_i)_{i=1}^n$, the empirical Rademacher complexity is defined as
    \begin{align*}
    \mc{R}_\mathcal{Z}(\Fcal)=\mathop{\Ebb}_{\theta_1, \theta_2, \ldots, \theta_n}\left[\sup_{f\in\mathcal{F}}\left(\sum_{i=1}^n\theta_i f(Z_i)\right)\right]
    \end{align*}
where $\{\theta_i\}_{i=1}^n$ are independent Rademacher random variables.
The Rademacher complexity is defined as
\begin{align*}
\mc{R}_n(\Fcal)=\mathop{\Ebb}_{Z_1,\ldots,Z_n}\left[R_\mathcal{Z}(\Fcal)\right].
    \end{align*}
\end{definition}
 
\begin{theorem}\label{theorem100}
Suppose $\mathcal{H}$ is a hypothesis space and $\ell$ is a loss function bounded in $[0,1]$.
Let $\delta$ be a positive number. Under our setting, with probability at least $1 - \delta$, for all $h \in \mathcal{H}$, we have
\begin{align*}
    \mathop{\Ebb}_{(X,Y) \sim \Pi}\left[\ell(h(X),Y)\right] \leq \frac{1}{n}\sum_{i=1}^n\ell(h(X_i),Y_i)
    +  2\mathcal{R}_\mathcal{Z}(\mathcal{F}_\ell) + \frac{1}{n}\sum_{i=1}^n\mu_i + 3\sqrt{||\Delta_n||^2_\infty\frac{\log(2/\delta)}{2n}}
\end{align*}
where $\mathcal{F}_\ell=\{(X,Y)\rightarrow\ell(h(X),Y)|h\in\mathcal{H}\}$ and $||\Delta_n||_\infty= 1+2\sum_{k=1}^n\phi(k)$.
\end{theorem}

\begin{Remark}
\citet{kuznetsov2017generalization} establish a generalization bound for asymptotically stationary processes.
However, their setting is different from ours.
They consider the scenario where data include $m$ independent blocks of mixing sequences of size $a$.
That is, the number of data points is $n = m a$.
They assume that the mixing sequences are asymptotically stationary.
More precisely, for a sequence $\{Z_i\}_{i=1}^\infty$, they define
\[
\beta(a):=\sup_t\Ebb \|\mathbb{P}_{t+a}(.|Z_1,...,Z_t)-\Pi \|_\text{TV}.
\]
The sequence $\{Z_i\}_{i=1}^\infty$ is asymptotically stationary if $\beta(a) \to 0$.
It is ready to see that $\mu_i \leq \beta(i)$ for any integer $i$. 
Therefore, the marginal distribution of an asymptotically stationary sequence converges to the target distribution $\Pi$.
In other words, their condition is more restricted compared to our condition.
Moreover, the convergence rate of their bound is $\mathcal{O}(1/\sqrt{m}$), which depends on the number of independent sequences.
So, their result is not applicable to the scenario we are considering in this paper, where data consists of only one mixing sequence.
On the other hand, they also require $\lim_{i \to \infty} i \beta(i) = 0$ while we only require $\lim_{i \to \infty}\mu_i = 0$.
Thus, their result requires the marginal distribution to converge to the target distribution at a faster rate than ours. 
\end{Remark}

Based on Theorem \ref{theorem100}, we can derive the following generalization bound for feed-forward neural networks:

\begin{theorem}\label{theorem200}
Assume that $\mu_n=\mathcal{O}\left(1/ \sqrt{n} \right)$ and $\phi(n)=\mathcal{O}\left( 1/n \right)$. 
Under our setting, with probability  at least $1-\delta$, for all margin $\gamma>0$ and network $F_{\mathcal{A},\boldsymbol{\sigma}}$, we have
\begin{multline*}
\mathbb{P}_{(X,Y) \sim \Pi} \left \{ \arg \max_j [F_{\mathcal{A}, \boldsymbol{\sigma}}(X)]_j\neq Y \right \}
\leq
\mathcal{L}_\mathcal{Z}(F_{\mathcal{A}, \boldsymbol{\sigma}}) \\
+\tilde{\mathcal{O}}
\left(\frac{\sqrt{\sum_i||X_i||^2_2}}{n\gamma}T_\mathcal{A}\log(W)+\sqrt{\frac{log(2/\delta)}{n}} \right)
\end{multline*}
where $T_\mathcal{A}=\left(\prod_{i=1}^Lp_i \|A_i\|_S\right)\left(\sum^L_{i=1}\left(\frac{\|A_i^T\|_{2,1}}{\|A_i\|_S}\right)^{2/3}\right)^{3/2}$.
Here, $f(x)=\tilde{\mathcal{O}}(g(x))$ means there exists $C, x_0 > 0$ such that $|f(x)| \leq C \log(x) g(x)$ for all $x \geq x_0$,  $||.||_S$ is the spectral norm, and $||.||_{p,q}$ is $(p, q)$-matrix norm, defined by $||A||_{p,q}=||(||A_{:,1}||_p,...,||A_{:,m}||_p)||_q$. 
\end{theorem}

\begin{Remark}\label{remark}
The generalization bound in \citet{bartlett2017spectrally} is a special case of Theorem \ref{theorem200} when data are iid.
\end{Remark}

\section{Proofs of main theorems}

In this section, we will provide proof of our main theorems.

\subsection{Proof of Theorem \ref{theorem100}}\label{sec3}

We first introduce some supporting Lemmas.

\begin{lemma}\label{theorem4}
Let $P$ be the distribution of a $\phi$-mixing sequence and $\Fcal$ be any class of functions.
Then 
\begin{equation*}
    \mathop{\Ebb}_{(Z_1,Z_2,\ldots,Z_n)\sim P}\left[ \sup_{f\in\mathcal{F}}\left[\frac{1}{n}\sum_{i=1}^nf(Z_i)-\mathop{\Ebb}_{(Z'_1,Z'_2,\ldots,Z'_n)\sim P}\left[\frac{1}{n}\sum_{i=1}^nf(Z'_i)\right]\right]\right]
    \leq 2\mathcal{R}_n(\mathcal{F}).
\end{equation*}
\end{lemma}
\begin{proof} We first rewrite the term inside of the first expectation:
\begin{multline*}
 \sup_{f\in\mathcal{F}}\left[\frac{1}{n}\sum_{i=1}^nf(Z_i)-\mathop{\Ebb}_{Z'_1,\ldots,Z'_n}\left[\frac{1}{n}\sum_{i=1}^nf(Z'_i)\right]\right]
\leq\frac{1}{n} \mathop{\Ebb}_{Z'_1,\ldots,Z'_n }\left[\sup_{f\in \mathcal{F}}\left(\sum_{i=1}^nf(Z_i)-\sum_{i=1}^n f(Z'_i)\right)\right]. 
\end{multline*}
Let $\{\theta_i\}_{i=1}^n$ be independent Rademacher random variables.
Taking the expectation with respect to $\{Z_i\}_{i=1}^n$ for both sides, we have
\begin{align*}
    &\mathop{\Ebb}_{Z_1,...,Z_n}\left[\sup_{f\in\mathcal{F}}\left[\frac{1}{n}\sum_{i=1}^nf(Z_i)-\mathop{\Ebb}_{Z'_1,...,Z'_n}\left[\frac{1}{n}\sum_{i=1}^nf(Z'_i)\right]\right]\right]\\
    &\leq \frac{1}{n}\mathop{\Ebb}_{Z_1,...,Z_n }\left[\mathop{\Ebb}_{Z'_1,...,Z'_n}\left[\sup_{f\in \mathcal{F}}\left(\sum_{i=1}^nf(Z_i)-\sum_{i=1}^n f(Z'_i)\right)\right]\right]\\
    &\leq\frac{1}{n}\mathop{\Ebb}_{Z_1,...,Z_n,Z'_1,...,Z'_n}\left[\mathop{\Ebb}_{\theta_j\overset{\mathrm{iid}}{\sim} \{+1;-1\}}\left[\sup_{f\in \mathcal{F}}\left(\sum_{i=1}^n\theta_if(Z_i)-\sum_{i=1}^n\theta_i f(Z'_i)\right)\right]\right]\\
    &\leq \mathop{\Ebb}_{Z_1,...,Z_n,Z'_1,...,Z'_n,\theta_1,...,\theta_n}\left[\sup_{f\in \mathcal{F}}\left(\frac{1}{n}\sum_{i=1}^n\theta_if(Z_i)-\frac{1}{n}\sum_{i=1}^n\theta_i f(Z'_i)\right)\right]\\
    &\leq \mathop{\Ebb}_{Z_1,...,Z_n,Z'_1,...,Z'_n,\theta_1,...,\theta_n}\left[\sup_{f\in \mathcal{F}}\left(\frac{1}{n}\sum_{i=1}^n\theta_if(Z_i)\right)+\sup_{f\in \mathcal{F}}\left(\frac{1}{n}\sum_{i=1}^n\theta_i f(Z'_i)\right)\right]\\
    &=2\mathcal{R}_n(\mathcal{F}).
\end{align*}
\end{proof}

\begin{lemma}\label{theorem32}\citep{mohri2010stability}
Let $\Phi:Z^n\rightarrow\Rbb$ be a measurable function that is $c$-Lipschitz with respect to the Hamming metric for some $c>0$ and let $\{Z_i\}_{i=1}^n$ be a $\phi$-mixing sequence. Then, for any $\epsilon > 0$, the following inequality holds:
\begin{align*}
    \Pbb\left[\big|\Phi(Z_1,\ldots,Z_n)-\Ebb[\Phi(Z_1,\ldots,Z_n)]\big|\geq\epsilon\right]\leq 2\exp\left(\frac{-2\epsilon^2}{nc^2\|\Delta_n\|^2_\infty}\right),
\end{align*}
where $\|\Delta_n\|^2_\infty= 1+2\sum_{k=1}^n\phi(k)$.

\end{lemma}

\begin{lemma}\label{theorem19}
Let $\mathcal{Z}=\{Z_i\}_{i=1}^n$ be a $\phi$-mixing sequence and $\Fcal$ be any class of functions bounded in $[0,1]$. Then, with probability at least $1 - \delta$, we have
\begin{align*}
    \Ebb\left[\frac{1}{n}\sum_{i=1}^nf(Z_i)\right]-\frac{1}{n}\sum_{i=1}^n f(Z_i)
    \leq 2\mathcal{R}_\mathcal{Z}(\mathcal{F})+3\sqrt{||\Delta_n||^2_\infty\frac{\log(2/\delta)}{2n}}, \quad \forall f \in \mathcal{F}.
\end{align*}
\end{lemma}
\begin{proof}
Define 
\begin{align*}
    g(Z_1, \ldots ,Z_n)\triangleq\sup_{f\in\mathcal{F}}\left[\Ebb\left[\frac{1}{n}\sum_{i=1}^nf(Z_i)\right]-\frac{1}{n}\sum_{i=1}^n f(Z_i)\right].
\end{align*}
We first show that $g$ is $\frac{1}{n}$-Lipschitz with respect to the Hamming distance. 
For any $\mathcal{Z}=(Z_1,\ldots,Z_n)$ and $\mathcal{Z}'=(Z'_1,\ldots,Z'_n)$, we have
\begin{align*}
    &|g(\mathcal{Z})-g(\mathcal{Z}')| \nonumber \\
    &=\left|\sup_{f\in\mathcal{F}}\left[\Ebb\left[\frac{1}{n}\sum_{i=1}^nf(Z_i)\right]-\frac{1}{n}\sum_{i=1}^n f(Z_i)\right]-\sup_{f\in\mathcal{F}}\left[\Ebb\left[\frac{1}{n}\sum_{i=1}^nf(Z'_i)\right]-\frac{1}{n}\sum_{i=1}^n f(Z'_i)\right]\right| \\
    &\leq\sup_{f\in\mathcal{F}}\left|\frac{1}{n}\sum_{i=1}^n f(Z_i)-\frac{1}{n}\sum_{i=1}^n f(Z'_i)\right| \leq \frac{\|\mathcal{Z} - \mathcal{Z}'\|_H}{n}
\end{align*}
where $\|\cdot\|_H$ is the Hamming distance.
The last inequality holds since $f$ bound in $[0,1]$.
Since $g$ is $\frac{1}{n}$-Lipschitz with respect to the Hamming distance, we apply Lemma \ref{theorem32} to obtain:
\begin{align*}
    \Pbb\left[g(\mathcal{Z})\geq\mathop{\Ebb}_{Z_1,Z_2, \ldots,Z_n}[g]+\epsilon\right]\leq  \exp\left(\frac{-2n\epsilon^2}{\|\Delta_n\|^2_\infty}\right).
\end{align*}
Applying Lemma \ref{theorem4}, we get $\mathop{\Ebb}_{Z_1,Z_2,\ldots,Z_n}[g]\leq 2\mathcal{R}_n(\mathcal{F}).$
We will show that $\mathcal{R}_\mathcal{Z}(\Fcal)$ is also $\frac{1}{n}$-Lipschitz with respect to the Hamming distance. 
Indeed, using similar arguments, we have
\begin{align*}
    &\left|
    \mathop{\Ebb}_{\theta_i}\left[\sup_{f\in\mathcal{F}}\left[\frac{1}{n}\sum_{i=1}^n \theta_if(Z_i)\right]\right]-\mathop{\Ebb}_{\theta_i}\left[\sup_{f\in\mathcal{F}}\left[\frac{1}{n}\sum_{i=1}^n \theta_if(Z'_i)\right]\right]
    \right| \\
    &\leq \mathop{\Ebb}_{\theta_i}\left|\sup_{f\in\mathcal{F}}\left[\frac{1}{n}\sum_{i=1}^n \theta_if(Z_i)\right]-\sup_{f\in\mathcal{F}}\left[\frac{1}{n}\sum_{i=1}^n \theta_if(Z'_i)\right]\right| \\
    &\leq \mathop{\Ebb}_{\mu_i}\left|\sup_{f\in\mathcal{F}}\left[\frac{1}{n}\sum_{i=1}^n \theta_if(Z_i)-\frac{1}{n}\sum_{i=1}^n \theta_if(Z'_i)\right]\right| \leq \frac{\|\mathcal{Z} - \mathcal{Z}'\|_H}{n}.
\end{align*}
We can thus apply Lemma \ref{theorem32}:
\begin{align*}
    \Pbb\left[\mathop{\Ebb}_{Z_1,Z_2,\ldots,Z_n}[\mathcal{R}_\mathcal{Z}(\Fcal)]\geq\mathcal{R}_\mathcal{Z}(\Fcal)+\epsilon\right]\leq \exp\left(\frac{-2n\epsilon^2}{||\Delta_n||^2_\infty}\right).
\end{align*}
We set $\epsilon=\sqrt{||\Delta_n||^2_\infty\frac{\log(2/\delta)}{2n}}$. Then with probability at least $1-\delta$,
\begin{align*}
    g
    \leq\Ebb[g]+\epsilon
    \leq 2\mathcal{R}_n(\Fcal)+\epsilon
    \leq 2(\mathcal{R}_\mathcal{Z}(\Fcal)+\epsilon)+\epsilon
    =2\mathcal{R}_\mathcal{Z}(\Fcal)+3\epsilon.
\end{align*}    
\end{proof}
\begin{lemma}\label{lemma3}
    Let $f$ is bounded function in $[0;1]$. Let $\{Z_i\}_{i=0}^n$ be a non-stationary $\phi$-mixing sequence such that the marginal distributions converge to a target distribution $\Pi$ with rate $\mu_n$. Then
    \begin{align*}
        \left|\Ebb\left[\frac{1}{n}\sum_{i=1}^nf(Z_i)\right]-\mathop{\Ebb}_{Z\sim \Pi}\left[f(Z)\right]\right|\leq \frac{1}{n}\sum_{i=1}^n\mu_i.
    \end{align*}
\end{lemma}
\begin{proof}
    For any $i$, 
    \begin{align}
    \Ebb\left[f(Z_i)\right]-\Ebb\left[f(Z)\right]=\int_\Omega f(z).(f_{\Pbb_i}(z)-f_\Pi(z))dz.\label{83}
    \end{align}
    Define $A=\{z\in\Omega|f_{\Pbb_i}(z)<f_\Pi(z)\}$ and $B=\{z\in\Omega|f_{\Pbb_i}(z)>f_\Pi(z)\}$.
    We rewrite Eq.(\ref{83}) 
    \begin{align*}
    \Ebb\left[f(Z_i)\right]-\Ebb\left[f(Z)\right]
    =\int_{A} f(z).(f_{\Pbb_i}(z)-f_\Pi(z))dz+\int_{B} f(z).(f_{\Pbb_i}(z)-f_\Pi(z))dz.
    \end{align*}
    For the first term, 
    \begin{align*}
        0\geq\int_{A} f(z).(f_{\Pbb_i}(z)-f_\Pi(z))dz\geq\int_{A} (f_{\Pbb_i}(z)-f_\Pi(z))dz\geq -\mu_i.
    \end{align*}
    For the second term, 
    \begin{align*}
        0\leq\int_{B} f(z).(f_{\Pbb_i}(z)-f_\Pi(z))dz\leq\int_{B} (f_{\Pbb_i}(z)-f_\Pi(z))dz\leq \mu_i.
    \end{align*}
    Then
    \begin{align*}
        -\mu_i\leq\Ebb\left[f(Z_i)\right]-\Ebb\left[f(Z)\right]\leq\mu_i.
    \end{align*}
    we have
    \begin{align*}
        \left|\Ebb\left[\frac{1}{n}\sum_{i=1}^nf(Z_i)\right]-\mathop{\Ebb}_{Z\sim \Pi}\left[f(Z)\right]\right|\leq \frac{1}{n}\sum_{i=1}^n\left|\Ebb\left[f(Z_i)\right]-\mathop{\Ebb}_{Z\sim \Pi}\left[f(Z)\right]\right|\leq\frac{1}{n}\sum_{i=1}^n\mu_i.
    \end{align*}
\end{proof}

Theorem \ref{theorem100} is a direct consequence of Lemmas \ref{theorem19} and \ref{lemma3}.

\subsection{Proof of Theorem \ref{theorem200}} 

Theorem \ref{theorem200} can be achieved by combining Theorem \ref{theorem100} and the proof technique of \citet{bartlett2017spectrally}.
Denote $\Fcal_{\mathcal{A},\boldsymbol{\sigma},\gamma}=\{(x,y)\rightarrow\ell_\gamma(\Mcal(F_{\mathcal{A},\boldsymbol{\sigma}}(x),y))\}$.
Applying Theorem \ref{theorem100}, we have
\begin{equation}
\mathcal{L}(F_{A,\sigma}) \leq \mathcal{L}_{\mathcal{Z}}(F_{\mathcal{A},\boldsymbol{\sigma}}) + 2 \mathcal{R}_{\mathcal{Z}}(\mathcal{F}_{\mathcal{A},\boldsymbol{\sigma},\gamma}) + \frac{1}{n}\sum_{i=1}^n\mu_i+3\sqrt{||\Delta_n||^2_\infty\frac{\log(2/\delta)}{2n}}
\label{eqn:NN}
\end{equation}
with probability at least $1 - \delta$.

Next, we introduce some supporting Lemmas.

\begin{lemma}[Lemma A.4 in \citet{bartlett2017spectrally}]\label{lemma4}
For any $f:\Rbb^d\rightarrow\Rbb^k$ and every $\gamma>0$, we have 
\[
\mathbb{P}_{(X,Y) \sim \Pi} \left \{ \arg \max_j [F_{\mathcal{A}, \boldsymbol{\sigma}}(X)]_j\neq Y \right \} \leq \mathcal{L}(F_{A,\sigma}).
\]
\end{lemma}

\begin{lemma}\label{theorem8}
Assume that $\sqrt{\sum_i||X_i||^2_2}\leq B$. 
For all feed-forward neural network $F_{\mathcal{A},\boldsymbol{\sigma}}:\Rbb^d\rightarrow\Rbb^k$ such that $||A_i||_\sigma\leq s_i$ and $||A_i^T||_{2,1}\leq b_i$, we have
\[
\mathcal{R}_{\mathcal{Z}}(F_{\mathcal{A},\boldsymbol{\sigma},\gamma}) \leq \frac{4}{\sqrt{n}^3}+\frac{36B\ln(2W)\ln(n)}{\gamma n} \left(\sum^L_{i=1}\left(\frac{b_i}{s_i}\right)^{2/3}\right)^{3/2} \left(\prod_{i=1}^Ls_ip_i\right).
\]
\end{lemma}
\begin{proof}
Using the same argument of the proof of Lemma A.8 in \citet{bartlett2017spectrally}, we obtain:
\begin{align*}
    \mathcal{R}_{\mathcal{Z}}(F_{\mathcal{A},\boldsymbol{\sigma},\gamma})
    \leq \inf_{a>0}\left(\frac{4\alpha}{\sqrt{n}}+\frac{12}{n} \int_{\alpha}^{\sqrt{n}}\sqrt{\frac{R}{\epsilon^2}} d\epsilon \right)
    = \inf_{a>0}\left(\frac{4\alpha}{\sqrt{n}} +\ln(\sqrt{n}/\alpha)\frac{12\sqrt{R}}{n} \right);
\end{align*}
where 
\[
R=\frac{4B^2\ln(2W^2)}{\gamma^2}\left(\sum^L_{i=1}\left(\frac{b_i}{s_i}\right)^{2/3}\right)^{3} \left(\prod_{i=1}^Ls_ip_i\right)^2.
\]
The desired bound may be obtained by setting $\alpha=1/n$.
\end{proof}

Combining equation \eqref{eqn:NN}, Lemma \ref{lemma4}, and Lemma \ref{theorem8}, we get the following Lemma
\begin{lemma} \label{lem:NN}
With probability at least $1-\delta$ over a non-stationary $\phi$-mixing sequence $\mathcal{Z}=((X_i,Y_i))^n_{i=1}$ with $\sqrt{\sum_i||X_i||^2_2}\leq B$, for all feed-forward neural network $F_{\mathcal{A},\boldsymbol{\sigma}}:\Rbb^d\rightarrow\Rbb^k$ such that $||A_i||_\sigma\leq s_i$ and $||A_i^T||_{2,1}\leq b_i$, we have
\begin{multline*}
\mathbb{P}_{(X,Y) \sim \Pi} \left \{ \arg \max_j [F_{\mathcal{A}, \boldsymbol{\sigma}}(X)]_j\neq Y \right \} \leq \mathcal{L}_{\mathcal{Z}}(F_{\mathcal{A},\boldsymbol{\sigma}}) + \frac{1}{n}\sum_{i=1}^n\mu_i+3\sqrt{||\Delta_n||^2_\infty\frac{\log(2/\delta)}{2n}} \\
+ \frac{8}{\sqrt{n}^3} + \frac{72B\ln(2W)\ln(n)}{\gamma n} \left(\sum^L_{i=1}\left(\frac{b_i}{s_i}\right)^{2/3}\right)^{3/2} \left(\prod_{i=1}^Ls_ip_i\right).
\end{multline*}
\end{lemma}
Now, we can achieve Theorem \ref{theorem200} by utilizing Lemma \ref{lem:NN} to derive a union bound over the parameter space and input space. This can be done by following the same steps as the proofs of Lemma A.9 in \citet{bartlett2017spectrally}. 
Finally, we note that $\mu_n=\mathcal{O}\left(1/\sqrt{n} \right)$ and $\phi(n)=\mathcal{O}\left( 1/n \right)$.
Therefore, 
\[\frac{1}{n}\sum_{i=1}^n\mu_i=\mathcal{O}\left(\frac{1}{\sqrt{n}} \right), \quad ||\Delta_n||_\infty=\mathcal{O}\left(\log n \right).
\]

\section{Discussion and conclusion}
In this paper, we propose a generalization bound of feed-forward neural networks for the nonstationary $\phi$-mixing sequences using Rademacher complexity. 
We first derive a generalization bound for bounded loss on a general hypothesis space when data are nonstationary and $\phi$-mixing.
Our result allows data to converge to the target distribution at a slower rate compared to \citet{kuznetsov2017generalization}.
Moreover, the generalization bound in \citet{kuznetsov2017generalization} does not work for our setting where data include only one mixing sequence.
Using our new bound, we establish a generalization bound of feed-forward neural networks, including the result of \citet{bartlett2017spectrally} for iid data as a special case. 
A future research direction is extending our generalization bound beyond mixing data.
Alternative options include data generated from a dynamical system \citep{ho2023adaptive}, evolutionary data \citep{ho2013asymptotic}, and data from infectious disease epidemics \citep{ho2018direct}.
Another direction is to develop a generalization bound for other types of deep neural networks.
This requires new bounds for the Rademacher complexity of these neural networks.

\section*{Acknowledgement}
LSTH was supported by the Canada Research Chairs program, the NSERC Discovery Grant RGPIN-2018-05447, and the NSERC Discovery Launch Supplement DGECR-2018-00181. We want to thank the University of Science, Vietnam National University Ho Chi Minh City, and AISIA Research Lab for supporting us in this project. 

\bibliographystyle{chicago}
\bibliography{mybibfile}
\end{document}